\documentclass[11pt]{article} 

\usepackage{amsmath,amsfonts,bm}

\def\eqref#1{equation~\ref{#1}}

\def\1{\bm{1}}

\DeclareMathAlphabet{\mathsfit}{\encodingdefault}{\sfdefault}{m}{sl}
\SetMathAlphabet{\mathsfit}{bold}{\encodingdefault}{\sfdefault}{bx}{n}

\newcommand{\E}{\mathbb{E}}

\newcommand{\R}{\mathbb{R}}

\usepackage{natbib}
\usepackage{booktabs}       
\usepackage{amsfonts}       
\usepackage{amsmath}        
\usepackage{graphicx}       
\usepackage{amsthm}         
\usepackage{amssymb}        
\usepackage{multirow}       
\usepackage{makecell}       
\usepackage[vlined,linesnumbered,ruled,resetcount]{algorithm2e} 
\usepackage[colorlinks,linkcolor=magenta,filecolor=blue,citecolor=blue,urlcolor=blue]{hyperref} 
\usepackage[top=1in, left=1in, right=1in, bottom=1in]{geometry} 
\usepackage{subfigure}      
\usepackage{pifont}         
\usepackage{mathtools}      
\usepackage{stmaryrd}       
\usepackage[T1]{fontenc}    
\usepackage{authblk}        
\usepackage{enumitem}       
\usepackage{algorithmic}    
\usepackage[utf8]{inputenc} 
\usepackage{nicefrac}       
\usepackage{microtype}      
\usepackage{xcolor}         
\usepackage{multicol}       
\usepackage{lscape}         
\usepackage{bbm}            
\usepackage{subcaption}     
\usepackage{wrapfig}        
\usepackage{xspace}         

\newtheorem{theorem}{Theorem}[section]

\newtheorem{condition}[theorem]{Condition}

\theoremstyle{definition}
\newtheorem{definition}[theorem]{Definition}

\newcommand{\wt}{\widetilde}

\renewcommand{\epsilon}{\varepsilon}
\renewcommand{\phi}{\varphi}

\renewcommand{\tilde}{\wt}

\theoremstyle{plain}
\newtheorem{objective}[theorem]{Objective}
\theoremstyle{definition}
\newcommand{\minmax}{\mathsf{Min}\text{-}\mathsf{Max}}
\newcommand{\nnz}{\mathbb{NNZ}}

\makeatletter
\newcommand*{\RN}[1]{\expandafter\@slowromancap\romannumeral #1@}
\makeatother

\makeatletter
\newcommand{\printfnsymbol}[1]{%
  \textsuperscript{\@fnsymbol{#1}}%
}
\makeatother

\title{Weighted Diversified Sampling for Efficient Data-Driven Single-Cell Gene-Gene Interaction Discovery}

\author[1]{Yifan Wu}
\author[2]{Yuntao Yang}
\author[1]{Zirui Liu}
\author[2]{Zhao Li}
\author[1]{Khushbu Pahwa}
\author[2]{Rongbin Li}
\author[2]{Wenjin Zheng\thanks{Corresponding authors} }
\author[1]{Xia Hu\printfnsymbol{1}}
\author[3]{Zhaozhuo Xu \printfnsymbol{1}}
\affil[1]{Rice University}
\affil[2]{UTHealth Houston}
\affil[3]{Stevens Institute of Technology}

\begin{document}

\maketitle

\begin{abstract}
    Gene-gene interactions play a crucial role in the manifestation of complex human diseases. Uncovering significant gene-gene interactions is a challenging task. Here, we present an innovative approach utilizing data-driven computational tools, leveraging an advanced Transformer model, to unearth noteworthy gene-gene interactions. Despite the efficacy of Transformer models, their parameter intensity presents a bottleneck in data ingestion, hindering data efficiency.  To mitigate this, we introduce a novel weighted diversified sampling algorithm. This algorithm computes the diversity score of each data sample in just two passes of the dataset, facilitating efficient subset generation for interaction discovery. Our extensive experimentation demonstrates that by sampling a mere 1\% of the single-cell dataset, we achieve performance comparable to that of utilizing the entire dataset.
\end{abstract}

\section{Introduction}
Gene-gene interactions play a crucial role in the manifestation of complex human diseases, including multiple sclerosis~\citep{brassat2006multifactor,motsinger2007complex,slim2022systematic}, pre-eclampsia~\citep{li2022identifying,diab2021molecular,williams2011genetics,oudejans2008placental}, and Alzheimer's disease~\citep{ghebranious2011pilot,hohman2016discovery}. Computational tools equipped with machine learning~(ML) prove effective in uncovering these significant gene interactions~\citep{mckinney2006machine,cui2022gene,yuan2021deep,wei2024self,upstill2013machine}. By learning an ML model on massive single-cell transcriptomic data, we can identify gene-gene interactions associated with complex but common human diseases. Existing models rely on prior knowledge such as transcription factors (TF)~\citep{wang2019systems, bbab142, 10.1093/bib/bbab325, Shu2021} or existing gene-gene interaction (GGI) networks  ~\citep{bbaa303, doi:10.1073/pnas.1911536116}, to infer new relationships. Although GGI networks and TFs are crucial for mapping biological processes, they frequently suffer from high false-positive rates and biases, particularly in large-scale in vitro experiments \citep{bmcbioinformatics2020, jxb2021}. In response to these challenges, we propose that gene-gene interactions can be uncovered using purely data-driven methods.

\textbf{The Rise of Transformers on Single-Cell Transcriptomic Data.} Recent advances in natural language processing, particularly the development of Transformer models~\citep{vaswani2017attention}, have demonstrated significant potential in biological data analysis ~\citep{Hao2023.05.29.542705,Theodoris2023,Bian2024.01.25.577152,cui2024scgpt}. Transformer models are known for their ability to capture the dependencies between gene expressions. The information fused through the self-attention mechanism~\citep{vaswani2017attention} is particularly suited for analyzing the intricate relationships in single-cell transcriptomic data. On the other hand, Transformer models also demonstrated superior performance when we scaled up their parameter size~\citep{Hao2023.05.29.542705}. This scaling capacity raises the researcher's interest in training and deploying parameter-intensive Transformer models, denoted as single-cell foundation models~\citep{cui2024scgpt}. We would like to take this advantage for better gene-gene interaction discovery by identifying feature interactions within Transformer models.


\textbf{Data-Driven Gene-Gene Interaction via Attention.} In this work, we would like to advance the gene-gene interaction discovery with the Transformer models that have demonstrated superior performance on single-cell transcriptomic data. We see the self-attention mechanism ~\citep{vaswani2017attention} as a pathway to facilitate the modeling of gene-gene interactions. In single-cell foundation models, the input to the model is a bag of $m$ gene expressions for a single cell. Next, in each layer and each head of the Transformer, there will be an attention map with shape $m\times m$ generated for this cell. Each entry of this attention map represents the interaction between two genes in this layer and this head. Assuming that we have a perfect Transformer that takes a cell gene expressions and correctly predicts if it is infected by a disease, we view the attention map of this cell as a strong indicator of disease-oriented gene-gene interactions.

\textbf{Efficiency Challenge in Data Ingestion.} Despite the transformative capabilities of Transformer models, one significant challenge remains: the efficient ingestion and processing of massive volumes of single-cell transcriptomic data. We are utilizing Transformer models with parameter sizes that exceed the hardware capacity, particularly that of the graphics processing unit (GPU). As a result, given a pre-trained Transformer, we have to perform batch-size computation on a massive single-cell transcriptomic dataset for computing gene-gene interactions through attention maps. This batch-size computation significantly enlarges the total execution time for scientific discovery. Moreover, the hardware in the real-world deployment environment for gene-gene interaction detection may have even more limited resources. Therefore, the current computational framework cannot support gene-gene interaction discovery on real-world single-cell transcriptomic datasets.

\textbf{Our Proposal: Two-Pass Weighted Diversified Sampling.}
In this paper, we introduce a novel weighted diversified sampling algorithm. This randomized algorithm computes the diversity score of each data sample in just two passes of the dataset. The proposed algorithm is highly memory-efficient and requires constant memory that is independent of the cell dataset size. Our theoretical analysis suggests that this diversity score estimates the density of the $\minmax$ kernel defined on the cell-level gene expressions, which provides the foundation and justification of the proposed strategy. Through extensive experiments, we demonstrate how the proposed sampling algorithm facilitates efficient subset generation for interaction discovery. The results show that by sampling a mere 1\% of the single-cell dataset, we can achieve performance comparable to that of utilizing the entire dataset.

\textbf{Our Contributions.} We summarize our contributions as fellows.
\begin{itemize}[nosep, leftmargin=*]
    \item We introduce a computational framework that advances the discovery of significant gene-gene interactions with CelluFormer, our proposed Transformer model that is trained on single-cell transcriptomic data.
    \item We pinpoint the challenge in data ingestion for the data-driven gene-gene interaction. Moreover, we argue that we should perform diversified sampling that selects a representative subset of single-cell transcriptomics data to fulfill the objective.
    \item We develop a diversity score for every cell in the dataset based on the $\minmax$ kernel density. Moreover, we perform a randomized algorithm that efficiently estimates the $\minmax$ kernel density for each cell. Furthermore, we use the estimated density to generate an effective subset for gene-gene interaction.
\end{itemize}

\section{Data-Driven Single-Cell Gene-Gene Interaction Discovery}\label{sec:data-driven_gene-gene}
In this section, we propose a computing framework to perform gene-gene interaction discovery on single-cell transcriptomic data. We start by introducing the format of single-cell transcriptomic data. Next, we propose the formulation of our CelluFormer model tailored to single-cell data. Next, we present our multi-cell-type training to build an effective transformer model on single-cell data. Finally, given a pre-trained transformer, we showcase how to perform gene-gene interaction discovery by analyzing the attention maps. 
\subsection{Single-Cell Transcriptomic Data}\label{sec:single-cell_data}
Single-cell transcriptomic is a technology that profiles gene expression at the individual cell level. The profiled results, namely single-cell transcriptomic data, provide a unique landscape of gene expressions. In contrast to traditional bulk RNA-seq analysis, single-cell transcriptomic data allows for cell-level sequencing, which captures the variability between individual cells \citep{bbaa303}. Leveraging this high-resolution data allows scientists to gain insights into developmental processes, disease mechanisms, and cellular responses to environmental changes.

The single-cell transcriptomic data can be formulated as a set of high-dimensional and sparse feature vectors. We denote a single-cell transcriptomic dataset at $X$, where each cell $x\in X$ is a sparse vector with dimensionality $V\in \mathbb{N}_{+}$. Here $V$ represents the total number of genes we can observe in $X$. Since cell $x\in \R^V$ is a sparse vector, we can represent $x$ as a set $\{(i_1,v_1), (i_2,v_2), \cdots, (i_k,v_k)\}$. In this set, every tuple $(i,v)$ represents the expression of gene $i\in [V]$ with expression level $v\in \R$. Besides we can also denote cell $x$ as $[x_1,x_2,\cdots, x_V]$, where most of the $x_i$s are zeros.

In this data formulation, single-cell transcriptomic data for each cell is represented as a set of gene expressions, with different cells expressing varying genes. Additionally, even when two cells express the same gene, their expression levels may differ. Our research objective is to identify gene-gene interactions within the vocabulary $V$ that drive complex biological processes and disease mechanisms.

\subsection{CelluFormer: A Single-Cell Transformer}\label{sec:set_transformer}
Here, we propose our Transformer architecture, CelluFormer, to learn gene-gene interactions within single-cell transcriptomic data. Based on the set formulation of single-cell transcriptomic data, we believe that the order of genes is arbitrary and biologically meaningless. 
Similar to scGPT \citep{cui2024scgpt}, and scFoundation \citep{Hao2024}, our method adopts a permutation-invariant design. We define our permutation-invariant condition as follows.

\begin{condition}\label{condition:permute}
Let $X$ denote a single-cell transcriptomic dataset. Given a single-cell data of cell $x\in X$, denoted as a set $\{(i_1,v_1), (i_2,v_2), \cdots, (i_k,v_k)\}$, a function $f:X\to \R$ should satisfy that, for any permutation $\pi$, $f(x)=f(\pi(x))$.
\end{condition}

\begin{wrapfigure}{r}{0.42\textwidth}
\vspace{-4mm}
  \begin{center}
    \includegraphics[width=0.4\textwidth]{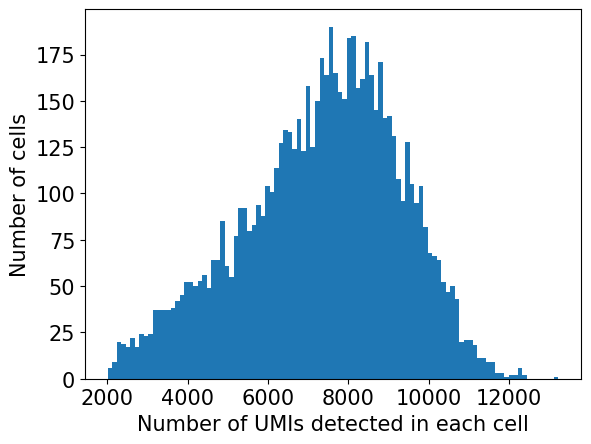}
  \end{center}
  \caption{Distribution of Sequence Lengths in L6\_CT Cell Type Data.}\label{fig:seq_len}
  \vspace{-3mm}
\end{wrapfigure}

We see Condition~\ref{condition:permute} as a fundamental difference between the proposed Transformer and the sequence Transformers~\citep{vaswani2017attention} widely used in natural language processing.  For sequence Transformers, we have to ingest sequential masks during the training to ensure that the current token does not interact with the future token. Additionally, during the inference, the sequence Transformer should perform a step-by-step generation for each token. As a result, the sequence Transformer does not satisfy Condition~\ref{condition:permute}. Moreover, the difference between CelluFormer and a vision Transformer~\citep{dosovitskiy2020image} is that the vision Transformer has a fixed sequence length for every input data sample. However, the number of genes expressed in each cell can vary a lot. For example, according to Figure~\ref{fig:seq_len}, the number of genes expressed in a single cell can be up to 12,000 or more. Thus, we utilize a padding mask for the classification downstream task. Additional details regarding the implementation of CelluFormer are provided in Appendix \ref{app:exp_details}.


\subsection{Multi-Cell-Type Training of CelluFormer}
\begin{wraptable}{r}{0.6\textwidth}
\vspace{-4mm}
\caption{Performance comparison of models on neuronal cell dataset.}
\label{tab:model_performance}
\scriptsize
\centering
\begin{tabular}{cccc}
\midrule
\textbf{Model} & \textbf{Training Dataset} & \textbf{F1 Score} & \textbf{Accuracy} \\
\midrule
\multirow{7}{*}{MLP} & Pax6 & 78.91 & 82.71 \\
 & L5\_ET & 62.02 & 73.31 \\
 & L6\_CT & 91.14 & 92.01 \\
 & L6\_IT\_Car3 & 95.34 & 95.51 \\
 & L6b & 86.01 & 88.76 \\
 & Chandelier & 81.66 & 84.56 \\
 & L5\_6\_NP & 89.33 & 90.42 \\
 & All Neuronal Cell Types & 97.23 & 97.25 \\
\midrule
CelluFormer & All Neuronal Cell Types & \textbf{98.12} & \textbf{98.12} \\
\midrule
\end{tabular}%
\vspace{-4mm}
\end{wraptable}

We observe that there is a significant performance difference between Transformer models if we feed them with different styles of single-cell transcriptomic data.  It is known that cells can be categorized into different types based on their functionality. For instance, neuronal cells represent the cell types that fire electric signals called action potentials across a neural network~\citep{levitan2015neuron}. Our study suggests that Transformers should be trained on single-cell transcriptomic data from various cell types to achieve better performance. We showcase an example in Table~\ref{tab:model_performance}. We train a Transformer model to classify whether a cell is an Alzheimer's disease-infected cell or not. According to our study, CelluFormer proposed in Section~\ref{sec:set_transformer} trained on neuronal cells outperforms traditional multilayer perceptron (MLP) with downstream training on a single cell type. However, we do not see this gap when we perform training of CelluFormer on a single cell type. As a result, we see that the Transformers generally prefer massive exposure to the single-cell transcriptomic data.



\subsection{Gene-Gene Interaction Discovery via Attention Maps}\label{sec:gene-gene_attn}
In this paper, we would like to accomplish the following objective.

\begin{objective}[Gene-gene interaction discovery]\label{obj:gene_gene_discover}
    Let $X$ denote a single-cell transcriptomic dataset. Let $\mathcal{V}$ denote the genes expressed in at least one $x\in X$.  Let $f:X\to\R$ denote a permutation invariant (see Condition~\ref{condition:permute}) CelluFormer. $f$ can successfully predict whether any $x\in X$ is infected by disease $D$. We would like to find a gene-gene pair $(v_1,v_2)$ that contributes the most to $f$'s performance in $X$. Here $v_1,v_2\in \mathcal{V}$.
\end{objective}

\begin{figure}[ht]
  \centering
  \includegraphics[width=1\textwidth]{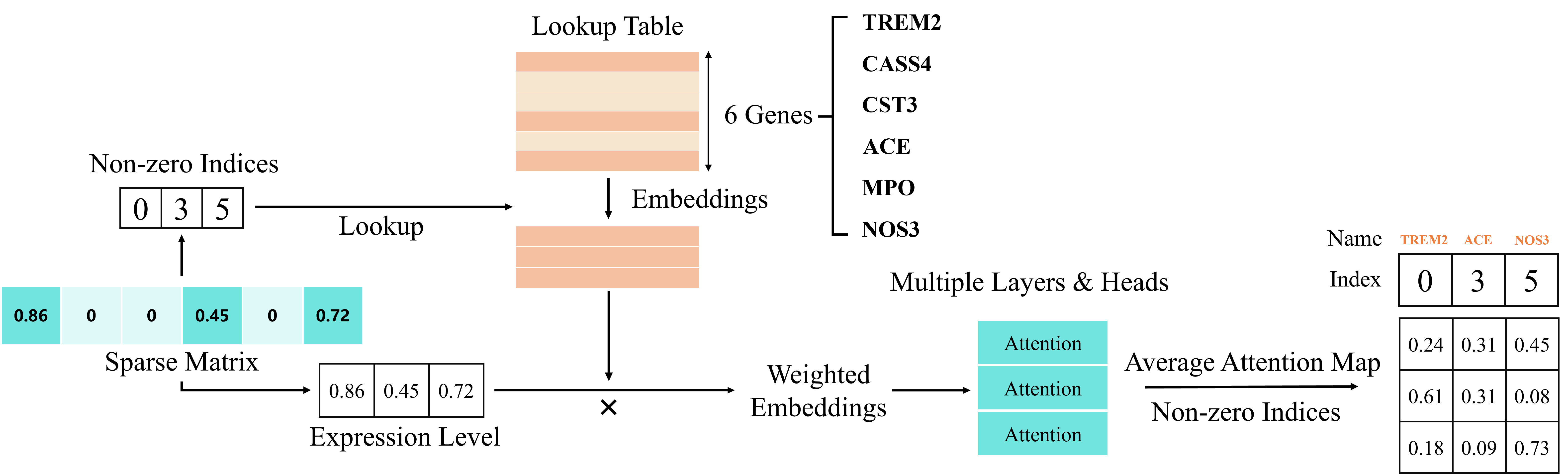}
  \caption{Gene-gene interaction modeling with attention maps.}
  \label{fig:attention_map}
\end{figure}

\begin{figure}[ht]
  \centering
  \includegraphics[width=0.9\textwidth]{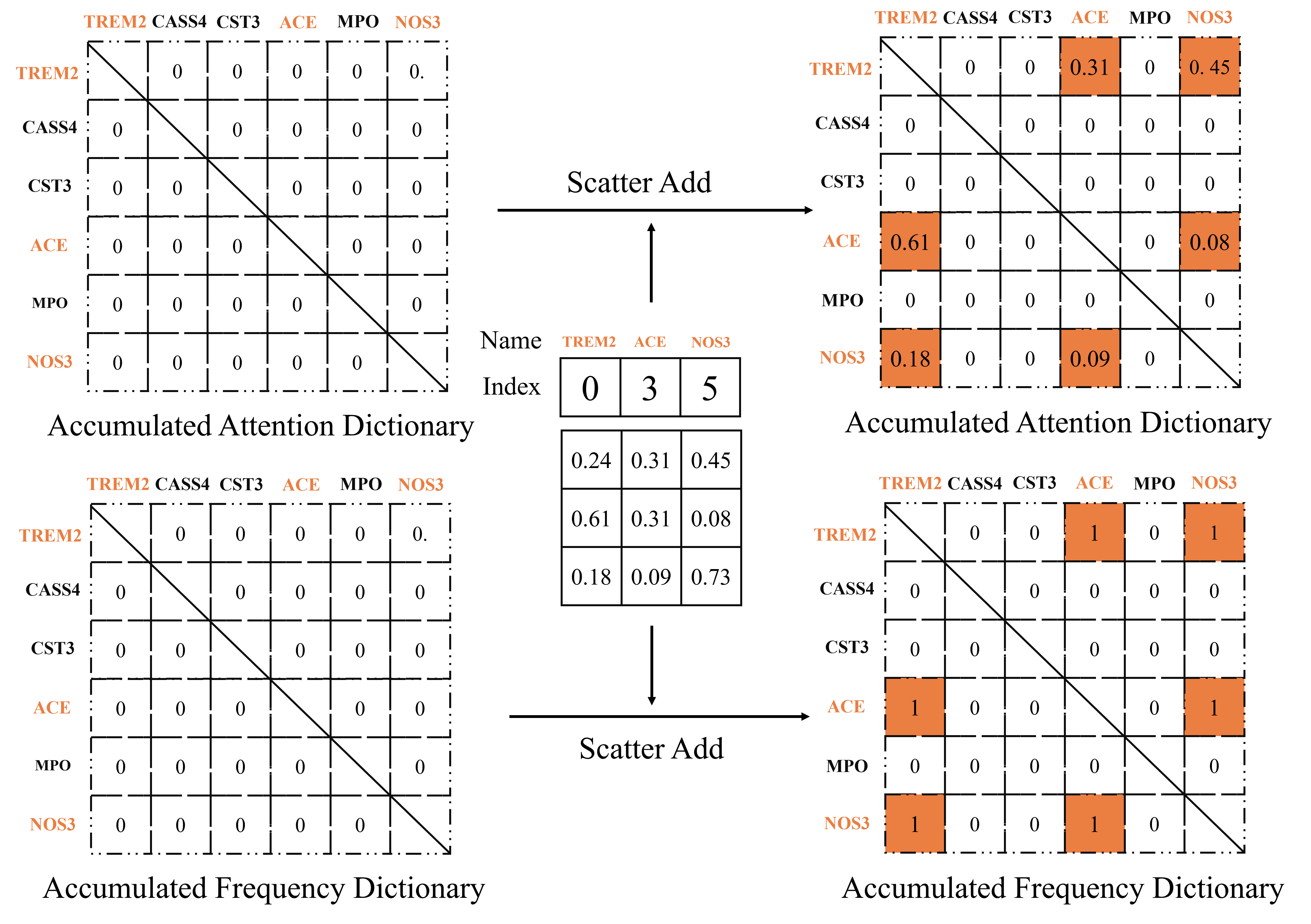}
  \caption{Accumulating multiple cells' average attention maps.}
  \label{fig:accumulate_attn}
\end{figure}

We see the self-attention mechanism of Transformers on a cell's set style gene expressions as a pathway to model gene-gene interactions. CelluFormer takes a cell $x$'s gene expressions and produces an attention map $A_{i,j}\in \R^{m\times m}$ at encoder block $i$ and attention head $j$. Here $m$ represents the number of genes expressed in cell $x$. Since Transformer architecture uses the Softmax function to produce $A_{i,j}$, we can view the $p$th row of  $A_{i,j}$ as the interaction between gene $p$ and all other genes in $x$. As a result, an attention map is a natural indicator of gene-gene interactions. Moreover, 
if we have a perfect Transformer that takes a cell $x$ gene expressions and correctly predicts if it is infected by a disease, we view the attention map of this cell as an indicator of disease-oriented gene-gene interactions. Following this path, we propose a gene-gene interaction modeling approach as illustrated in Figure~\ref{fig:attention_map}. For each cell $x$, we represent it as a set and generate a bag of embeddings from the gene embedding table. Next, we use the expression levels of each gene as a scaling factor for each gene's embedding. Next, we take the average attention maps of all layers and all heads to obtain a gene-gene interaction map in this cell. 

In Objective~\ref{obj:gene_gene_discover}, we would like to see not only the gene-gene interactions just for cell $x$ but also the statistical evidence of how two genes interact in the dataset $X$. As a result, we propose to accumulate multiple cells' averaged attention maps as illustrated in Figure~\ref{fig:accumulate_attn}. For $X$, we initialize $Z_0\in 0^{V\times V}$ matrix as the overall attention map before aggregation and $M_0 \in 0^{V\times V}$ as the overall frequency dictionary before aggregation. Next, for each cell $x$ in the dataset, we remove its diagonal value in its averaged attention map as it represents self-interaction. Next, we perform scatter addition operations that merge $x$'s averaged attention map back to $Z_0$. We let $Z_{ij}$ add the interaction value of gene $v_i$ and $v_j$ in the average attention map of cell $x$ obtained in the Transformer model. Simultaneously, to eliminate the dataset bias of expressed genes, we count the number of appearances for each gene pair in the dataset. Once again, we perform scatter addition to record the counts back to $M_0$. This is done by updating $M_0$ through scatter addition, where $M_{ij} = M_{ij} + 1$ for every occurrence of the gene pair $(v_i, v_j)$ in the dataset. Finally, we rank the off-diagonal values in $Z$ where $Z_{ij} \leftarrow \frac{Z_{ij}}{M_{ij}}$ to retrieve the top gene-gene interaction.

\section{Weighted Diversified Sampling}
In this section, we start by showcasing the data-efficiency problem when we use the trained CelluFormer for gene-gene interaction discovery. Following this, we define a diversity score for each cell in the dataset and propose a two-pass randomized algorithm to efficiently compute it. Lastly, we propose a weighted diversified sampling strategy on massive single-cell data. 

\subsection{Data-Intensive Computation for Gene-Gene Interaction Discovery}
As illustrated in Section~\ref{sec:gene-gene_attn}, once we have a pre-trained CelluFormer that can successfully predict whether a cell is infected by a disease or not with its gene expressions, we can perform gene-gene interaction discovery by passing massive cells into this model and get the accumulated attention map as Figure~\ref{fig:accumulate_attn}. However, this process requires data-intensive computation. For every cell in the dataset, we first need to compute the average attention map as illustrated in Figure~\ref{fig:attention_map}. Next, we perform aggregations as shown in Figure~\ref{fig:accumulate_attn}. It is known that CelluFormer uses plenty of trainable parameters to achieve good disease infection classification performance. As a result, the computation complexity for generating a cell's averaged attention map is expensive. Moreover, since the attention map for cell $x$ is $m\times m$, where $m$ is the number of genes expressed in $x$. According to Figure~\ref{fig:seq_len}, we see that $m$ can be 12,000 or more. These giant attention maps consume the limited high bandwidth memory (HBM) in the graphics processing unit. Therefore, we have to perform batch-wise computation on a massive cell dataset for computing gene-gene interactions. 
Moreover, given the scale of the dataset, \textit{any sampling algorithm with a runtime that grows exponentially with the dataset size is impractical.}

\begin{algorithm}[h]
   \caption{Two-Pass Algorithm for Estimating $\minmax$ Density}
\begin{algorithmic}
\label{alg:two_pass_diverse}
   \STATE {\bfseries Input:} Cell dataset $X$, 0-bit CWS function family $\mathcal{H}$ (see Definition~\ref{def:cws}), Hash range $B$, Rows $R$
   \STATE {\bfseries Output:} $\minmax$ density set $w$ for every $x\in X$.
   \STATE {\bfseries Initialize:} $ A \leftarrow 0^{R \times B}$
   \STATE Generated $R$ independent 0-bit CWS functions ${h_1,\dots,h_R}$ from $\mathcal{H}$ with range $B$ at Random.\\\COMMENT{\textcolor{blue}{We set $R=O(\log{|X|})$ following the theoretical analysis of Definition~\ref{def:cws}}}
   \STATE $W\leftarrow \emptyset$
   \FOR{$x\in X$}
    \FOR{$r = 1 \rightarrow R$}
      \STATE $A_{r, h_r(x)} += 1$
     \ENDFOR
    \ENDFOR
    \FOR{$x\in X$}
    \FOR{$r = 1 \rightarrow R$}
      \STATE $w_x\leftarrow w_x+A_{r, h_r(x)}$
     \ENDFOR
     \STATE $w_x \leftarrow w_x/R$ \COMMENT{\textcolor{blue}{$w_x$ is the estimated $\minmax$ density for $x$.}}
     \STATE $W\leftarrow \{w_x\}$
    \ENDFOR
    \RETURN  $W$
\end{algorithmic}
\end{algorithm}

\subsection{Two-Pass Randomized Algorithm for Computing \texorpdfstring{$\minmax$}{~} Density}\label{sec:WDS definition}
In this work, we would like to address this data-efficiency challenge by raising and asking the following research question: \textit{Can we find a representative and small subset from the large cell dataset and still perform successful gene-gene interaction discovery?} Moreover, we would like the procedure for finding this small subset as efficient as possible.

We would like to answer this question by proposing a diversity score of a cell in the dataset. To begin with, we introduce the $\minmax$ similarity between two cell's gene expressions.
\begin{definition}[$\minmax$ Similarity]\label{def:min-max}
    Given two cell's gene expressions, denoted as $x, y\in \R^{V}$ (see Section~\ref{sec:single-cell_data}), we define their $\minmax$ similarity as:
    $\minmax (x,y) = \frac{\sum_{i}^{V} \min (x_i,y_i)}{\sum_{i}^{V} \max (x_i,y_i)}$.
\end{definition}

According to the definition, $\minmax (x,y)\in [0,1]$. Higher $\minmax$ means that two cell's gene expressions are closer to each other. $\minmax$ is widely viewed as a kernel~\citep{li2015min,li2021consistent,li2021rejection} in statistical machine learning. In this paper, we would like to define a kernel density on top of the $\minmax$ similarity.

\begin{definition}[$\minmax$ Density]\label{def:min-max_density}
    Given a cell dataset $X$, for every $q\in X$, we define its $\minmax$ density as:
    $\mathcal{K}(q)=\sum_{x\in X} \phi(q,x)$,
    where $\phi(q,x):\R\to \R$ is a monotonic increasing function along with $\minmax(q,x)$ similarity defined in Definition~\ref{def:min-max}.
\end{definition}

We view $\minmax(q)$ density as an indicator of how diverse $q$ is in $X$. Smaller $\minmax(q)$ means that all other $x\in X$ may be less similar to $q$, making $q$ a unique cell. On the other hand, higher $\minmax(q)$ means that $X$ has some cells that have similar gene expressions with $q$, making $q$ less unique. However, to compute $\minmax(q)$ for every $q\in X$ following  Definition~\ref{def:min-max_density}, we have to compute all pairwise $\minmax (x,y)$ for any $x,y \in X$, which results in an unaffordable $O(n^2\nnz(X))$ time complexity, where $n$ is the size of $X$ and $\nnz(X)$ is the maximum possible number of genes expressed in a cell $x\in X$. To reduce this $n^2$ computation, we propose a randomized algorithm that takes advantage of 0-bit consistent weighted sampling (CWS)~\citep{li20150} hash functions.

\begin{definition}[0-bit Consistent Weighted Sampling Hash Functions~\citep{li20150,li2021consistent}]\label{def:cws}
    Let $\mathcal{H}$ denote a randomized hash function family. 
    If we pick a $h\in \mathcal{H}$ at random, for any two cell expressions $x,y\in \R^V$, we have $\Pr[h(x)=h(y)] = \minmax(x,y)+o(1)$.
    Here every $h\in \mathcal{H}$ is a hash function that maps any $x\in X$ to an integer in $[0,B)$. We denote $B$ as the hash range.
\end{definition}
Here the $o(1)$ is a minor additive term with complex form. For simplicity, we refer the readers to \citep{li2021consistent}, Theorem 4.4 for more details. 

This work presents an efficient randomized algorithm that estimates $\minmax$ density $\mathcal{K}(q)$ (see Definition~\ref{def:min-max_density}) for every $q\in X$. As showcased in Algorithm~\ref{alg:two_pass_diverse},  we initialize an array $A$ with all values as zeros. Next, we conduct a pass over $X$. In this pass, for every $x\in X$, we compute its hash values after $R$ independent hash functions. Next, we increment $A_{r,h_r(x)}$ with $1$. After this pass, we take another pass at the dataset, for every $x\in X$, we take an average over the $A_{r,h_r(x)}$ and build a density score $w_x$. 
We would like to highlight that Algorithm~\ref{alg:two_pass_diverse} requires only two linear scans of the dataset. The time complexity for this algorithm is $O(n\nnz(X))$, which is linear to the dataset.  Moreover, we show that Algorithm~\ref{alg:two_pass_diverse} produces an estimator to $\minmax$ density.

\begin{theorem}[$\minmax$ Density Estimator, informal version of Theorem~\ref{thm:min-max_density:formal}]\label{thm:min-max_density:informal}
    Given a cell dataset $X$, for every $q\in X$, we compute $w_q$ following Algorithm~\ref{alg:two_pass_diverse}. Next, we have
    $\E[w_q] = \sum_{x\in X} (\minmax(x,q)+o(1))$,
    where $\minmax$ is the $\minmax$ similarity defined in Definition~\ref{def:min-max}.
    As a result, $w_q$ is an estimator for $\minmax$ density $\mathcal{K}(q)$ defined in Definition~\ref{def:min-max_density} with $\phi(q,x)=\minmax(x,q)+o(1)$.
\end{theorem}

We provide the proof of Theorem~\ref{thm:min-max_density:informal} in the supplementary materials.

\subsection{Weighted Diversified Sampling with Inverse \texorpdfstring{$\minmax$}{~} Density}

We propose to use the inverse form of $\minmax$ density in Definition~\ref{def:min-max_density} as a score for diversity. We define it as normalized inverse $\minmax$ density as below.

\begin{definition}[Inverse $\minmax$ Density (IMD)]\label{def:inverse_minmiax_density}
    Given a cell dataset $X$, for every $q\in X$, we define its normalized inverse $\minmax$ density as 
    $\mathcal{I}(q)=\mathsf{Softmax}(1/\mathcal{K}(q))$,
    where $\mathcal{K}(q)$ is the $\minmax$ diversity for $q$ in Definition~\ref{def:min-max_density}, $\mathsf{Softmax}$ is the softmax function that takes over all cells in $X$.
\end{definition}

We view the IMD $\mathcal{I}(q)\in [0,1]$ as a monotonic increasing function for the diversity of $q$. Higher $\mathcal{I}(q)$ means that all other $x\in X$ may be less similar to $q$, making $q$ a unique cell. Moreover, IMD can be directly used as a sample probability to generate a representative subset of $X$ for Objective~\ref{obj:gene_gene_discover}. Given $X$, we perform sampling without replacement to generate a subset $X_{\mathsf{sub}}\subset X$, where $x\in X$ has the sampling probability $\mathcal{I}(x)$. The advantages of sampling with IMD (see Definition~\ref{def:inverse_minmiax_density}) can be summarized as follows.
\begin{itemize}[nosep, leftmargin=*]
    \item The IMD $\mathcal{I}(q)$ can be an effective indicator for how diverse $q$ is in dataset $X$.
    \item Computing IMD is an efficient one-shot preprocessing process with just two linear scans of $X$ with time complexity $O(n\nnz(X))$, where $n$ and $\nnz(X)$ is defined in Section~\ref{sec:WDS definition}. 
    \item The memory complexity of computing IMD is $O(RB)$, which can be viewed as constant since it is independent of $n$ and $\nnz(X)$. 
\end{itemize}

In the following section, we would like to evaluate the empirical performance of IMD in selecting a representative subset out of a massive single-cell transcriptomic dataset while still maintaining effective performance in data-driven gene-gene interaction discovery powered by CelluFormer.

\begin{definition}[Estimated Interaction Score with WDS]\label{def:estimated_interact_score}
Let $Z_x(v_i,v_j)$ denote the interaction value of gene $v_i$ and $v_j$ in the average attention map of cell $x$ obtained in the CelluFormer. For dataset $X$, we perform a sampling where each cell $x\in X$ is sampled with probability $\mathcal{I}(x)$ (see Definition~\ref{def:inverse_minmiax_density}) and  get a subset $X_{s}$. Next, we define the estimated interaction score between gene $v_i$ and $v_j$ learned from $X$ as:
\begin{align*}
    \tilde{Z}(v_i,v_j) = \frac{\sum_{x\in X_{s}}Z_x(v_i,v_j)\cdot \mathcal{I}(x)}{\sum_{x\in X_{s}}\mathcal{I}(x)},
\end{align*}
where $\tilde{Z}(v_i,v_j)$ is an unbiased estimator for the expectation of $Z(v_i,v_j)$ in distribution with density $\mathcal{I}(x)$. Formally, 
\begin{align*}
    \E[\tilde{Z}(v_i,v_j)] = \E_{x\sim \mathcal{I}(x)} [Z_x(v_i,v_j)],
\end{align*}
\begin{align*}
    \mathbf{Var}[\tilde{Z}(v_i,v_j)] 
        = ~\frac{\sum_{x\in X_s}\mathcal{I}(x)^2}{(\sum_{x\in X_s}\mathcal{I}(x))^2}\mathbf{Var}_{x\sim \mathcal{I}(x)}[Z_x(v_i,v_j)].
\end{align*}
\end{definition}

\section{Experiment}

In this section, we want to validate the effectiveness of our gene-gene interaction pipeline as well as the two-pass diversified sampling algorithm~\ref{alg:two_pass_diverse}. There are a few research questions we want to answer:

\begin{itemize}[nosep, leftmargin=*]
    \item \textbf{RQ1:} How does the proposed Transformer-based computing framework introduced in Section~\ref{sec:data-driven_gene-gene} perform in gene-gene interaction discovery?

    \item \textbf{RQ2:} How does the $\minmax$ density estimated by two-pass diversified sampling Algorithm~\ref{alg:two_pass_diverse} characterize the diversity of a cell in the whole dataset? Is this estimated  $\minmax$ density useful?

    \item \textbf{RQ3:} How does the estimated  $\minmax$ density perform in improving data-efficiency of gene-gene interaction discovery? How is the quality of the subset sampled according to the estimated  $\minmax$ density? 
\end{itemize}

\subsection{Settings}

\textbf{Dataset:} For the training dataset, we employ the Seattle Alzheimer’s Disease Brain Cell Atlas (SEA-AD) \citep{Gabitto2023}, which includes single nucleus RNA sequencing data of 36,601 genes (as 36,601 features) from 84 senior brain donors exhibiting varying degrees of Alzheimer’s Disease (AD) neuropathological changes.
By providing extensive cellular and genetic data, SEA-AD enables in-depth exploration of the cellular heterogeneity and gene expression profiles associated with AD. To facilitate a comparative analysis between AD-affected and non-AD brains, we select cells from 42 donors classified within the high-AD category and 9 donors from the non-AD category, based on their neuropathological profiles. This selection criterion ensures a robust comparison, aiding in the identification of gene-gene interactions linked to AD progression \citep{Gabitto2023}.
The dataset is comprehensively annotated, covering 1,240,908 cells across 24 distinct cell types. We selected 18 neuronal cell types as our final training dataset since we believe neuronal cells are more likely to reveal explainable gene-gene interactions that are related to Alzheimer’s Disease compared to non-neuronal cells. To better detect expression relationships among genes, we apply the Seurat Transformation Function \citep{seurat} to eliminate the problem of sequence depth difference.

\noindent \textbf{Model:} For the SEA-AD dataset, we designed a CelluFormer model as explained in \ref{sec:set_transformer} to predict labels indicative of Alzheimer's disease conditions. Further details can be found in the Appendix \ref{app:exp_details}.

\noindent \textbf{Baselines:}
Our proposed algorithm leverages the attention maps of the Transformer models. Accordingly, we compare our method with three statistical methods, Pearson Correlation, CS-CORE, and Spearman's Correlation  \citep{freedman2007statistics, Su2023, DeSmet2010}. While these methods are widely adopted by biologists for gene co-expression analysis, gene co-expression values alone do not provide information about the relationship between gene pairs and Alzheimer's Disease. To identify gene-gene interactions relevant to Alzheimer's Disease, we apply these methods to subsets containing disease and non-disease cells respectively, and calculate their gene co-expression values. The difference in co-expression values between disease and non-disease cells is then used as a final score to rank the gene pairs. We also present more experiments in Appendix \ref{app:more_exp} that demonstrate how Transformers aggregate data with varying labels.

Our baseline includes NID \citep{tsang2017detecting}, a traditional feature interpretation technique that extracts learned interactions from trained MLPs. NID identifies interacting features by detecting strongly weighted connections to a standard hidden unit in MLPs after training. We evaluated our CelluFormer model against the MLP model, with performance results presented in Table \ref{tab:model_performance}. 

Additionally, to comprehensively evaluate RQ1, we utilized two existing single-cell large foundation models to assess our algorithm. Specifically, we fine-tuned two foundation models, scFoundation \citep{Hao2024} and scGPT \citep{Cui2023}, to classify whether a cell is AD or non-AD (performance results are provided in Table \ref{tab:model_performance_complete}). We then applied our gene-gene interaction discovery pipeline using the attention maps of these foundation models.

In the sampling experiments, we compare WDS with uniform sampling since none of them requires preprocessing time exponential to the dataset size.

\noindent \textbf{Evaluation Metric:}
For a comprehensive evaluation encompassing the entire ranked list of gene-gene interactions, we utilized the Kolmogorov-Smirnov test, which was facilitated by the \texttt{GSEApy} package \citep{Fang2023} in Python. We select normalized enrichment score (NES) \citep{Subramanian2005} as our evaluation metric. 
The ground truth dataset is sourced from \textit{BioGRID} and \textit{DisGenet} \citep{Oughtred2019, gkw943}. 
For our experiments, we extract a subset of DisGenet that includes genes associated with Alzheimer's Disease. We then filter out genes in BioGRID that are not present in this DisGenet subset. Finally, we obtain a filtered BioGRID dataset containing only genes relevant to Alzheimer's Disease. We provide more explanations about our evaluation metrics in Appendix \ref{app:eval}.


\subsection{The Effectiveness of Transformers in Gene-Gene Interaction Discovery (RQ1)}

To evaluate the effectiveness of our Transformer-based framework for gene-gene interaction discovery, we performed feature selection across seven different cell types used as inference datasets. Additionally, we used a dataset encompassing all neuronal cell types to assess the overall performance of various models. As shown in Table~\ref{tab:RQ1_res}, Transformer-based methods, including CelluFormer, scGPT and scFoundation, significantly outperformed other baselines.
 
This result indicates that our proposed Transformer-based framework is more effective and stable at extracting general and global gene-gene interaction information. In addition, the foundation models, scGPT and scFoundation, achieved comparable performances with other baselines across some of the datasets. 
We attribute this outcome to two main factors. \textbf{Overfitting to Pretrained Knowledge}: A foundation model, particularly a large one, might have learned very general or specific knowledge during its pretraining phase. When fine-tuning for a specific task, the model might overfit the preexisting knowledge, leading to poor generalization of the new task data. \textbf{Mismatch Between Pretraining and Fine-Tuning Data}: If the data distribution for fine-tuning is significantly different from the data on which the foundation model was trained, the model might struggle to adapt, resulting in worse performance. A model trained from scratch on the specific task data may perform better as it directly optimizes for that data distribution.

\begin{table}[ht]
\caption{Performance comparison of models on neuronal cell data. To evaluate different models on datasets with varying sizes, we further select 7 neuronal cell types from all neuronal cell types. CelluFormer, scGPT, scFoundation, MLP, Pearson Correlation, Spearman's Correlation, and CS-CORE were tested on 8 different datasets to obtain their gene pair rankings.}
\label{tab:RQ1_res}
\centering
\resizebox{1\columnwidth}{!}{%
\begin{tabular}{lrrrrrrr}
\toprule
\multicolumn{1}{l}{\textbf{Dataset}} & \textbf{CelluFormer} & \textbf{scFoundation} & \textbf{scGPT} & \textbf{NID} & \textbf{Pearson} & \textbf{CS-CORE} & \textbf{Spearman} \\ \midrule
L5\_ET         & 1.15  & 1.04  & \textbf{1.23}  & 0.90  & 0.50  & 1.11 & 0.91 \\
L6\_CT         & 1.18  & 1.03  & 1.17  & \textbf{1.54}  & -0.21  & 1.06 & 0.72 \\
Pax6           & \textbf{1.25}  & 0.82  & 1.01  & 1.04  & 0.93  & 0.95 & 1.15 \\
L5\_6\_NP      & 1.21  & 1.06  & \textbf{1.50}  & 1.49  & 0.87  & 0.92 & 0.95 \\
L6b            & 1.13  & 0.99  &\textbf{ 1.23}  & 0.62  & 0.75  & 0.62 & 1.08 \\
Chandelier     & \textbf{1.17}  & 1.16  & 1.09  & 1.07  & 0.94  & 1.06 & 0.96 \\
L6\_IT\_Car3   & \textbf{1.22}  & 0.90  & 0.66  & 1.19  & 0.59  & 1.08 & 0.86 \\
All neuron data & \textbf{1.17}  & 1.02  & 0.99  & 0.86  & 1.01  & 1.06 & 1.04 \\
\bottomrule
\end{tabular}%
}

\end{table}

\subsection{Ablation Studies (RQ2 \& RQ3)}

We addressed these questions by comparing our weighted diversified sampling (WDS) method with uniform sampling across various sample sizes, ranging from 1\% to 10\% of the original dataset. We generated data subsets for each cell type using WDS and uniform sampling. We then applied our Transformer-based framework for feature selection at each sample size. Since CelluFormer consistently outperformed other baselines, we selected it as our base model. We repeated Each experiment five times and recorded the NES scores as the results. To evaluate the sampling methods, we calculated the average NES score across the five experiments. We also computed the Mean Square Error (MSE) between the NES scores from the sampling experiments and the ground truth derived from the entire dataset, as shown in Table \ref{tab:RQ1_res}. The evaluation results are presented in Table \ref{tab:Sampling_Res}. We note that WDS consistently produced higher NES scores compared to uniform sampling. As the sample size increased, the NES scores from uniform sampling began to converge with the ground truth. In contrast, the NES scores from WDS consistently remained close to the ground truth, even at smaller sample sizes. The result indicates that while WDS offers a significant advantage in small samples by enabling the Transformer to capture a broader range of genetic interactions, its benefits diminish as more data becomes available. Moreover, we find that for some cell types, smaller samples of data outperformed larger samples of data on NES. This suggests that: (1) single-cell transcriptomic data may contain noises that affect gene-gene interaction discovery, and, (2) some complex gene-gene interaction patterns in single-cell transcriptomic data cannot be interpreted directly through attention maps. We also provide a detailed study on the choice of parameter $R$ in Algorithm~\ref{alg:two_pass_diverse} in Appendix~\ref{sec:exp_param_r}.

\begin{table}[ht]
\scriptsize
\centering
\caption{Evaluation Results for the transformer over sample data. For each cell type, we performed 8 groups of down-sampling regarding 4 different sample sizes and 2 sampling methods. We let the transformer conduct inferences over the sample data and generate results.}
\label{tab:Sampling_Res}
\resizebox{0.7\columnwidth}{!}{%
\begin{tabular}{lcrrrrr}
\toprule
 &  & \multicolumn{2}{c}{Mean of NES} & \multicolumn{2}{c}{MSE of NES} \\ 
\multirow{-2}{*}{Dataset} & \multirow{-2}{*}{Sample Size} & \multicolumn{1}{r}{Uniform} & \multicolumn{1}{r}{WDS} & \multicolumn{1}{r}{Uniform} & \multicolumn{1}{r}{WDS} \\ \midrule
\multirow{4}{*}{L5\_ET} & 1\% & 0.90 & \textbf{0.95} & 0.0127 & \textbf{0.0082} \\
 & 2\% & 0.89 & \textbf{1.17} & 0.0131 & \textbf{0.0001} \\
 & 5\% & 1.02 & \textbf{1.19} & 0.0036 & \textbf{0.0003} \\
 & 10\% & 0.87 & \textbf{1.07} & 0.0158 & \textbf{0.0012} \\ \midrule
\multirow{4}{*}{L6\_CT} & 1\% & 0.85 & \textbf{1.19} & 0.0207 & \textbf{0.0000} \\
 & 2\% & 1.05 & \textbf{1.18} & 0.0030 & \textbf{4.30e-05} \\
 & 5\% & 0.93 & \textbf{1.23} & 0.0122 & \textbf{0.0006} \\
 & 10\% & 0.91 & \textbf{1.21} & 0.0136 & \textbf{0.0002} \\ \midrule
\multirow{4}{*}{Pax6} & 1\% & 0.94 & \textbf{1.08} & 0.0184 & \textbf{0.0053} \\
 & 2\% & 1.03 & \textbf{1.18} & 0.0098 & \textbf{0.0009} \\
 & 5\% & 0.98 & \textbf{1.20} & 0.0139 & \textbf{0.0004} \\
 & 10\% & 1.06 & \textbf{1.17} & 0.0072 & \textbf{0.0012} \\ \midrule
\multirow{4}{*}{L5\_6\_NP} & 1\% & 0.90 & \textbf{1.13} & 0.0192 & \textbf{0.0016} \\
 & 2\% & \textbf{1.15} & 1.11 & \textbf{0.0009} & 0.0021 \\
 & 5\% & 1.02 & \textbf{1.20} & 0.0076 & \textbf{4.54e-06} \\
 & 10\% & 1.01 & \textbf{1.17} & 0.0080 & \textbf{0.0004} \\ \midrule
\multirow{4}{*}{L6b} & 1\% & 0.79 & \textbf{1.17} & 0.0226 & \textbf{0.0004} \\
 & 2\% & 0.76 & \textbf{1.14} & 0.0266 & \textbf{0.0000} \\
 & 5\% & 0.88 & \textbf{1.20} & 0.0121 & \textbf{0.0009} \\
 & 10\% & 1.20 & \textbf{1.21} & \textbf{0.0010} & 0.0014 \\ \midrule
\multirow{4}{*}{L6\_IT\_Car3} & 1\% & 0.78 & \textbf{1.20} & 0.0384 & \textbf{0.0001} \\
 & 2\% & 0.87 & \textbf{1.15} & 0.0242 & \textbf{0.0011} \\
 & 5\% & 0.97 & \textbf{1.17} & 0.0123 & \textbf{0.0006} \\
 & 10\% & 0.97 & \textbf{1.18} & 0.0123 & \textbf{0.0003} \\ 
\bottomrule
\end{tabular}%
}
\vspace{-4mm}
\end{table}
\section{Related Work}
\textbf{Single-Cell Transformer Models.} Single-cell RNA sequencing (scRNA-seq), or single-cell transcriptomics, enables high-throughput insights into cellular systems, amassing extensive databases of transcriptional profiles across various cell types for the construction of foundational cellular models ~\citep{Hao2023.05.29.542705}. Recently, there has emerged a large number of transformer models pre-trained for single-cell RNA sequencing tasks, including scFoundation ~\citep{Hao2023.05.29.542705}, Geneformer ~\citep{Theodoris2023}, scMulan ~\citep{Bian2024.01.25.577152}, scGPT ~\citep{cui2024scgpt}. These foundation models have gained a progressive understanding of gene expressions and can build meaningful gene encodings over limited transcriptomic data. Yet, the previous work did not pay attention to pairwise gene-gene interactions. In our work, we would like to highlight a fundamental functionality of single-cell foundation models: we must use these models to perform data-driven scientific discovery.

\noindent \textbf{Randomized Algorithms for Efficient Kernel Density Estimation.}
Kernel density estimation (KDE) is a fundamental task in both machine learning and statistics. It finds extensive use in real-world applications such as outlier detection~\citep{luo2018arrays, coleman2020sub} and genetic abundance analysis~\citep{coleman2022one}. Recently, there has been a growing interest in applying hash-based estimators (HBE)\citep{cs17,biw19,srb+19,coleman2020sub,ss21} for KDE. HBEs leverage Locality Sensitive Hashing (LSH)\citep{indyk1998approximate,datar2004locality,li2019re} functions, where the collision probability of two vectors under an LSH function is monotonic relative to their distance measure. This property allows HBE to perform efficient importance sampling using LSH functions and hash table-type data structures. Furthermore, \citep{liu2024one} extend KDE algorithms as a sketch for the distribution. However, previous works have not considered LSH for weighted similarity. In this work, we focus on designing a new HBE that incorporates the $\minmax$ similarity~\citep{li2015min}, a weighted similarity measure.

\section{Conclusion}

Gene-gene interactions are pivotal in the development of complex human diseases, yet identifying these interactions remains a formidable challenge. In response, we have developed a pioneering approach that utilizes an advanced Transformer model to effectively reveal significant gene-gene interactions. Although Transformer models are highly effective, their extensive parameter requirements often impede efficient data processing. To overcome this limitation, we have introduced a weighted diversified sampling algorithm. This innovative algorithm efficiently calculates the diversity score of each data sample across just two passes of the dataset. With this method, we enable the rapid generation of optimized data subsets for interaction analysis. Our comprehensive experiments illustrate that by leveraging this method to sample a mere 1\% of the single-cell dataset, we can achieve results that rival those obtained using the full dataset, significantly enhancing both efficiency and scalability.

\newpage

\bibliographystyle{plainnat}
\bibliography{ref}

\begin{thebibliography}{67}
\providecommand{\natexlab}[1]{#1}
\providecommand{\url}[1]{\texttt{#1}}
\expandafter\ifx\csname urlstyle\endcsname\relax
  \providecommand{\doi}[1]{doi: #1}\else
  \providecommand{\doi}{doi: \begingroup \urlstyle{rm}\Url}\fi

\bibitem[Ata et~al.(2020)Ata, Wu, Fang, Ou-Yang, Kwoh, and Li]{bbaa303}
Sezin~Kircali Ata, Min Wu, Yuan Fang, Le~Ou-Yang, Chee~Keong Kwoh, and Xiao-Li Li.
\newblock {Recent advances in network-based methods for disease gene prediction}.
\newblock \emph{Briefings in Bioinformatics}, 22\penalty0 (4):\penalty0 bbaa303, 12 2020.
\newblock ISSN 1477-4054.
\newblock \doi{10.1093/bib/bbaa303}.
\newblock URL \url{https://doi.org/10.1093/bib/bbaa303}.

\bibitem[Backurs et~al.(2019)Backurs, Indyk, and Wagner]{biw19}
Arturs Backurs, Piotr Indyk, and Tal Wagner.
\newblock Space and time efficient kernel density estimation in high dimensions.
\newblock \emph{Advances in Neural Information Processing Systems}, 32, 2019.

\bibitem[Bian et~al.(2024)Bian, Chen, Dong, Li, Hao, Chen, Hu, Sun, Wei, and Zhang]{Bian2024.01.25.577152}
Haiyang Bian, Yixin Chen, Xiaomin Dong, Chen Li, Minsheng Hao, Sijie Chen, Jinyi Hu, Maosong Sun, Lei Wei, and Xuegong Zhang.
\newblock scmulan: a multitask generative pre-trained language model for single-cell analysis.
\newblock \emph{bioRxiv}, 2024.
\newblock \doi{10.1101/2024.01.25.577152}.
\newblock URL \url{https://www.biorxiv.org/content/early/2024/01/29/2024.01.25.577152}.

\bibitem[Brassat et~al.(2006)Brassat, Motsinger, Caillier, Erlich, Walker, Steiner, Cree, Barcellos, Pericak-Vance, Schmidt, et~al.]{brassat2006multifactor}
D~Brassat, Alison~A Motsinger, SJ~Caillier, HA~Erlich, K~Walker, LL~Steiner, BAC Cree, LF~Barcellos, MA~Pericak-Vance, S~Schmidt, et~al.
\newblock Multifactor dimensionality reduction reveals gene--gene interactions associated with multiple sclerosis susceptibility in african americans.
\newblock \emph{Genes \& Immunity}, 7\penalty0 (4):\penalty0 310--315, 2006.

\bibitem[Charikar and Siminelakis(2017)]{cs17}
Moses Charikar and Paris Siminelakis.
\newblock Hashing-based-estimators for kernel density in high dimensions.
\newblock In \emph{2017 IEEE 58th Annual Symposium on Foundations of Computer Science (FOCS)}, pages 1032--1043. IEEE, 2017.

\bibitem[Chen et~al.(2021{\natexlab{a}})Chen, Cheong, Lan, Zhou, Liu, Lyu, Cheung, and Zhang]{10.1093/bib/bbab325}
Jiaxing Chen, ChinWang Cheong, Liang Lan, Xin Zhou, Jiming Liu, Aiping Lyu, William~K Cheung, and Lu~Zhang.
\newblock {DeepDRIM: a deep neural network to reconstruct cell-type-specific gene regulatory network using single-cell RNA-seq data}.
\newblock \emph{Briefings in Bioinformatics}, 22\penalty0 (6):\penalty0 bbab325, 08 2021{\natexlab{a}}.
\newblock ISSN 1477-4054.
\newblock \doi{10.1093/bib/bbab325}.
\newblock URL \url{https://doi.org/10.1093/bib/bbab325}.

\bibitem[Chen et~al.(2021{\natexlab{b}})Chen, Cheong, Lan, Zhou, Liu, Lyu, Cheung, and Zhang]{Chen2021}
Jiaxing Chen, ChinWang Cheong, Liang Lan, Xin Zhou, Jiming Liu, Aiping Lyu, William~K Cheung, and Lu~Zhang.
\newblock Deepdrim: a deep neural network to reconstruct cell-type-specific gene regulatory network using single-cell rna-seq data.
\newblock \emph{Briefings in Bioinformatics}, 22\penalty0 (6), 2021{\natexlab{b}}.
\newblock \doi{10.1093/bib/bbab325}.
\newblock URL \url{https://doi.org/10.1093/bib/bbab325}.

\bibitem[Coleman et~al.(2019)Coleman, Shrivastava, and Baraniuk]{coleman2019race}
Benjamin Coleman, Anshumali Shrivastava, and Richard~G Baraniuk.
\newblock Race: Sub-linear memory sketches for approximate near-neighbor search on streaming data.
\newblock \emph{arXiv preprint arXiv:1902.06687}, 2019.

\bibitem[Coleman et~al.(2020)Coleman, Baraniuk, and Shrivastava]{coleman2020sub}
Benjamin Coleman, Richard Baraniuk, and Anshumali Shrivastava.
\newblock Sub-linear memory sketches for near neighbor search on streaming data.
\newblock In \emph{International Conference on Machine Learning}, pages 2089--2099. PMLR, 2020.

\bibitem[Coleman et~al.(2022)Coleman, Geordie, Chou, Elworth, Treangen, and Shrivastava]{coleman2022one}
Benjamin Coleman, Benito Geordie, Li~Chou, RA~Leo Elworth, Todd Treangen, and Anshumali Shrivastava.
\newblock One-pass diversified sampling with application to terabyte-scale genomic sequence streams.
\newblock In \emph{International Conference on Machine Learning}, pages 4202--4218. PMLR, 2022.

\bibitem[Cui et~al.(2023)Cui, Wang, Maan, Pang, Luo, and Wang]{Cui2023}
Haotian Cui, Chloe Wang, Hassaan Maan, Kuan Pang, Fengning Luo, and Bo~Wang.
\newblock scgpt: Towards building a foundation model for single-cell multi-omics using generative ai.
\newblock \emph{bioRxiv}, 2023.
\newblock \doi{10.1101/2023.04.30.538439}.
\newblock URL \url{https://www.biorxiv.org/content/early/2023/07/02/2023.04.30.538439}.

\bibitem[Cui et~al.(2024)Cui, Wang, Maan, Pang, Luo, Duan, and Wang]{cui2024scgpt}
Haotian Cui, Chloe Wang, Hassaan Maan, Kuan Pang, Fengning Luo, Nan Duan, and Bo~Wang.
\newblock scgpt: toward building a foundation model for single-cell multi-omics using generative ai.
\newblock \emph{Nature Methods}, pages 1--11, 2024.

\bibitem[Cui et~al.(2022)Cui, El~Mekkaoui, Reinvall, Havulinna, Marttinen, and Kaski]{cui2022gene}
Tianyu Cui, Khaoula El~Mekkaoui, Jaakko Reinvall, Aki~S Havulinna, Pekka Marttinen, and Samuel Kaski.
\newblock Gene--gene interaction detection with deep learning.
\newblock \emph{Communications Biology}, 5\penalty0 (1):\penalty0 1238, 2022.

\bibitem[Datar et~al.(2004)Datar, Immorlica, Indyk, and Mirrokni]{datar2004locality}
Mayur Datar, Nicole Immorlica, Piotr Indyk, and Vahab~S. Mirrokni.
\newblock Locality-sensitive hashing scheme based on p-stable distributions.
\newblock In \emph{Proceedings of the 20th {ACM} Symposium on Computational Geometry (SoCG)}, pages 253--262, Brooklyn, NY, 2004.

\bibitem[De~Smet and Marchal(2010)]{DeSmet2010}
Riet De~Smet and Kathleen Marchal.
\newblock Advantages and limitations of current network inference methods.
\newblock \emph{Nature Reviews Microbiology}, 8\penalty0 (10):\penalty0 717--729, October 2010.
\newblock \doi{10.1038/nrmicro2419}.
\newblock URL \url{https://www.nature.com/articles/nrmicro2419}.

\bibitem[Diab et~al.(2021)Diab, Barish, Dong, Zhao, Allington, Yu, Kahle, Brueckner, and Jin]{diab2021molecular}
Nicholas~S Diab, Syndi Barish, Weilai Dong, Shujuan Zhao, Garrett Allington, Xiaobing Yu, Kristopher~T Kahle, Martina Brueckner, and Sheng~Chih Jin.
\newblock Molecular genetics and complex inheritance of congenital heart disease.
\newblock \emph{Genes}, 12\penalty0 (7):\penalty0 1020, 2021.

\bibitem[Dosovitskiy et~al.(2020)Dosovitskiy, Beyer, Kolesnikov, Weissenborn, Zhai, Unterthiner, Dehghani, Minderer, Heigold, Gelly, et~al.]{dosovitskiy2020image}
Alexey Dosovitskiy, Lucas Beyer, Alexander Kolesnikov, Dirk Weissenborn, Xiaohua Zhai, Thomas Unterthiner, Mostafa Dehghani, Matthias Minderer, Georg Heigold, Sylvain Gelly, et~al.
\newblock An image is worth 16x16 words: Transformers for image recognition at scale.
\newblock In \emph{International Conference on Learning Representations}, 2020.

\bibitem[Erten et~al.(2011)Erten, Bebek, and Koyutürk]{Erten2011}
Sinan Erten, Gurkan Bebek, and Mehmet Koyutürk.
\newblock Vavien: an algorithm for prioritizing candidate disease genes based on topological similarity of proteins in interaction networks.
\newblock \emph{Journal of Computational Biology}, 18:\penalty0 1561--1574, 2011.
\newblock \doi{10.1089/cmb.2011.0178}.
\newblock URL \url{https://doi.org/10.1089/cmb.2011.0178}.

\bibitem[Fang et~al.(2023)Fang, Liu, and Peltz]{Fang2023}
Zhuoqing Fang, Xinyuan Liu, and Gary Peltz.
\newblock Gseapy: a comprehensive package for performing gene set enrichment analysis in python.
\newblock \emph{Bioinformatics}, 39, 2023.

\bibitem[Freedman et~al.(2007)Freedman, Pisani, and Purves]{freedman2007statistics}
David Freedman, Robert Pisani, and Roger Purves.
\newblock Statistics (international student edition).
\newblock \emph{Pisani, R. Purves, 4th edn. WW Norton \& Company, New York}, 2007.

\bibitem[Gabitto et~al.(2023)Gabitto, Travaglini, Rachleff, Kaplan, Long, Ariza, Ding, et~al.]{Gabitto2023}
Mariano~I. Gabitto, Kyle~J. Travaglini, Victoria~M. Rachleff, Eitan~S. Kaplan, Brian Long, Jeanelle Ariza, Yi~Ding, et~al.
\newblock Integrated multimodal cell atlas of alzheimer’s disease.
\newblock \emph{Research Square}, 2023.

\bibitem[Ghebranious et~al.(2011)Ghebranious, Mukesh, Giampietro, Glurich, Mickel, Waring, and McCarty]{ghebranious2011pilot}
Nader Ghebranious, Bickol Mukesh, Philip~F Giampietro, Ingrid Glurich, Susan~F Mickel, Stephen~C Waring, and Catherine~A McCarty.
\newblock A pilot study of gene/gene and gene/environment interactions in alzheimer disease.
\newblock \emph{Clinical Medicine \& Research}, 9\penalty0 (1):\penalty0 17--25, 2011.

\bibitem[Hao et~al.(2023)Hao, Gong, Zeng, Liu, Guo, Cheng, Wang, Ma, Song, and Zhang]{Hao2023.05.29.542705}
Minsheng Hao, Jing Gong, Xin Zeng, Chiming Liu, Yucheng Guo, Xingyi Cheng, Taifeng Wang, Jianzhu Ma, Le~Song, and Xuegong Zhang.
\newblock Large scale foundation model on single-cell transcriptomics.
\newblock \emph{bioRxiv}, 2023.
\newblock \doi{10.1101/2023.05.29.542705}.
\newblock URL \url{https://www.biorxiv.org/content/early/2023/06/21/2023.05.29.542705}.

\bibitem[Hao et~al.(2024)Hao, Gong, Zeng, Liu, Guo, Cheng, Wang, Ma, Zhang, and Song]{Hao2024}
Minsheng Hao, Jing Gong, Xin Zeng, Chiming Liu, Yucheng Guo, Xingyi Cheng, Taifeng Wang, Jianzhu Ma, Xuegong Zhang, and Le~Song.
\newblock Large-scale foundation model on single-cell transcriptomics.
\newblock \emph{Nature Methods}, 21\penalty0 (8):\penalty0 1481--1491, 2024.
\newblock ISSN 1548-7105.
\newblock \doi{10.1038/s41592-024-02305-7}.
\newblock URL \url{https://doi.org/10.1038/s41592-024-02305-7}.

\bibitem[Hohman et~al.(2016)Hohman, Bush, Jiang, Brown-Gentry, Torstenson, Dudek, Mukherjee, Naj, Kunkle, Ritchie, et~al.]{hohman2016discovery}
Timothy~J Hohman, William~S Bush, Lan Jiang, Kristin~D Brown-Gentry, Eric~S Torstenson, Scott~M Dudek, Shubhabrata Mukherjee, Adam Naj, Brian~W Kunkle, Marylyn~D Ritchie, et~al.
\newblock Discovery of gene-gene interactions across multiple independent data sets of late onset alzheimer disease from the alzheimer disease genetics consortium.
\newblock \emph{Neurobiology of aging}, 38:\penalty0 141--150, 2016.

\bibitem[Indyk and Motwani(1998)]{indyk1998approximate}
Piotr Indyk and Rajeev Motwani.
\newblock Approximate nearest neighbors: Towards removing the curse of dimensionality.
\newblock In \emph{Proceedings of the Thirtieth Annual {ACM} Symposium on the Theory of Computing (STOC)}, pages 604--613, Dallas, TX, 1998.

\bibitem[Kingma and Ba(2017)]{kingma2017adam}
Diederik~P. Kingma and Jimmy Ba.
\newblock Adam: A method for stochastic optimization, 2017.

\bibitem[Lee et~al.(2019)Lee, Lee, Kim, Kosiorek, Choi, and Teh]{lee2019set}
Juho Lee, Yoonho Lee, Jungtaek Kim, Adam Kosiorek, Seungjin Choi, and Yee~Whye Teh.
\newblock Set transformer: A framework for attention-based permutation-invariant neural networks.
\newblock In \emph{International conference on machine learning}, pages 3744--3753. PMLR, 2019.

\bibitem[Levitan and Kaczmarek(2015)]{levitan2015neuron}
Irwin~B Levitan and Leonard~K Kaczmarek.
\newblock \emph{The neuron: cell and molecular biology}.
\newblock Oxford University Press, USA, 2015.

\bibitem[Li(2015{\natexlab{a}})]{li20150}
Ping Li.
\newblock 0-bit consistent weighted sampling.
\newblock In \emph{Proceedings of the 21th ACM SIGKDD International conference on knowledge discovery and data mining}, pages 665--674, 2015{\natexlab{a}}.

\bibitem[Li(2015{\natexlab{b}})]{li2015min}
Ping Li.
\newblock Min-max kernels.
\newblock \emph{arXiv preprint arXiv:1503.01737}, 2015{\natexlab{b}}.

\bibitem[Li et~al.(2019)Li, Li, and Zhang]{li2019re}
Ping Li, Xiaoyun Li, and Cun-Hui Zhang.
\newblock Re-randomized densification for one permutation hashing and bin-wise consistent weighted sampling.
\newblock \emph{Advances in Neural Information Processing Systems}, 32, 2019.

\bibitem[Li et~al.(2021)Li, Li, Samorodnitsky, and Zhao]{li2021consistent}
Ping Li, Xiaoyun Li, Gennady Samorodnitsky, and Weijie Zhao.
\newblock Consistent sampling through extremal process.
\newblock In \emph{Proceedings of the Web Conference 2021}, pages 1317--1327, 2021.

\bibitem[Li et~al.(2022)Li, Liu, Whitehead, Li, Thierry, Le, and Winter]{li2022identifying}
Xiaomei Li, Lin Liu, Clare Whitehead, Jiuyong Li, Benjamin Thierry, Thuc~D Le, and Marnie Winter.
\newblock Identifying preeclampsia-associated genes using a control theory method.
\newblock \emph{Briefings in Functional Genomics}, 21\penalty0 (4):\penalty0 296--309, 2022.

\bibitem[Li and Li(2021)]{li2021rejection}
Xiaoyun Li and Ping Li.
\newblock Rejection sampling for weighted jaccard similarity revisited.
\newblock In \emph{Proceedings of the AAAI Conference on Artificial Intelligence}, volume~35, pages 4197--4205, 2021.

\bibitem[Liu et~al.(2024)Liu, Xu, Coleman, and Shrivastava]{liu2024one}
Zichang Liu, Zhaozhuo Xu, Benjamin Coleman, and Anshumali Shrivastava.
\newblock One-pass distribution sketch for measuring data heterogeneity in federated learning.
\newblock \emph{Advances in Neural Information Processing Systems}, 36, 2024.

\bibitem[Luo and Shrivastava(2018)]{luo2018arrays}
Chen Luo and Anshumali Shrivastava.
\newblock Arrays of (locality-sensitive) count estimators (ace) anomaly detection on the edge.
\newblock In \emph{Proceedings of the 2018 World Wide Web Conference}, pages 1439--1448, 2018.

\bibitem[Mahdavi and Lin(2007)]{bmcbioinformatics2020}
Mahmoud~A. Mahdavi and Yen-Han Lin.
\newblock False positive reduction in protein-protein interaction predictions using gene ontology annotations.
\newblock \emph{BMC Bioinformatics}, 8\penalty0 (1):\penalty0 262, 2007.
\newblock ISSN 1471-2105.
\newblock \doi{10.1186/1471-2105-8-262}.
\newblock URL \url{https://doi.org/10.1186/1471-2105-8-262}.

\bibitem[McKinney et~al.(2006)McKinney, Reif, Ritchie, and Moore]{mckinney2006machine}
Brett~A McKinney, David~M Reif, Marylyn~D Ritchie, and Jason~H Moore.
\newblock Machine learning for detecting gene-gene interactions: a review.
\newblock \emph{Applied bioinformatics}, 5:\penalty0 77--88, 2006.

\bibitem[Motsinger et~al.(2007)Motsinger, Brassat, Caillier, Erlich, Walker, Steiner, Barcellos, Pericak-Vance, Schmidt, Gregory, et~al.]{motsinger2007complex}
Alison~A Motsinger, David Brassat, Stacy~J Caillier, Henry~A Erlich, Karen Walker, Lori~L Steiner, Lisa~F Barcellos, Margaret~A Pericak-Vance, Silke Schmidt, Simon Gregory, et~al.
\newblock Complex gene--gene interactions in multiple sclerosis: a multifactorial approach reveals associations with inflammatory genes.
\newblock \emph{Neurogenetics}, 8:\penalty0 11--20, 2007.

\bibitem[Oudejans and Van~Dijk(2008)]{oudejans2008placental}
CBM Oudejans and M~Van~Dijk.
\newblock Placental gene expression and pre-eclampsia.
\newblock \emph{Placenta}, 29:\penalty0 78--82, 2008.

\bibitem[Oughtred et~al.(2019)Oughtred, Stark, Breitkreutz, Rust, Boucher, Chang, Kolas, et~al.]{Oughtred2019}
Rose Oughtred, Chris Stark, Bobby-Joe Breitkreutz, Jennifer Rust, Lorrie Boucher, Christie Chang, Nadine Kolas, et~al.
\newblock The {BioGRID} interaction database: 2019 update.
\newblock \emph{Nucleic Acids Research}, 47, 2019.

\bibitem[Paszke et~al.(2017)Paszke, Gross, Chintala, Chanan, Yang, DeVito, Lin, Desmaison, Antiga, and Lerer]{paszke2017automatic}
Adam Paszke, Sam Gross, Soumith Chintala, Gregory Chanan, Edward Yang, Zachary DeVito, Zeming Lin, Alban Desmaison, Luca Antiga, and Adam Lerer.
\newblock Automatic differentiation in pytorch.
\newblock 2017.

\bibitem[Piñero et~al.(2016)Piñero, Bravo, Queralt-Rosinach, Gutiérrez-Sacristán, Deu-Pons, Centeno, García-García, Sanz, and Furlong]{gkw943}
Janet Piñero, Àlex Bravo, Núria Queralt-Rosinach, Alba Gutiérrez-Sacristán, Jordi Deu-Pons, Emilio Centeno, Javier García-García, Ferran Sanz, and Laura~I. Furlong.
\newblock {DisGeNET: a comprehensive platform integrating information on human disease-associated genes and variants}.
\newblock \emph{Nucleic Acids Research}, 45\penalty0 (D1):\penalty0 D833--D839, 10 2016.
\newblock ISSN 0305-1048.
\newblock \doi{10.1093/nar/gkw943}.
\newblock URL \url{https://doi.org/10.1093/nar/gkw943}.

\bibitem[Rao et~al.(2018)Rao, VG, Joseph, Bhattacharya, and Srinivasan]{Rao2018}
Ananthakrishnan Rao, Sagar VG, Tony Joseph, Anand Bhattacharya, and R.~Srinivasan.
\newblock Phenotype-driven gene prioritization for rare diseases using graph convolution on heterogeneous networks.
\newblock \emph{BMC Medical Genomics}, 11:\penalty0 57, 2018.
\newblock \doi{10.1186/s12920-018-0372-8}.
\newblock URL \url{https://doi.org/10.1186/s12920-018-0372-8}.

\bibitem[Rasmussen and et~al.(2021)]{jxb2021}
Sebastian Rasmussen and et~al.
\newblock Predicting protein interactions in plants: High confidence comes at a cost.
\newblock \emph{Journal of Experimental Botany}, 73\penalty0 (12):\penalty0 3866--3876, 2021.
\newblock \doi{10.1093/jxb/erab332}.
\newblock URL \url{https://academic.oup.com/jxb/article/73/12/3866/6565416}.

\bibitem[Shu et~al.(2021)Shu, Zhou, Lian, Li, Zhao, Zeng, and Ma]{Shu2021}
Hantao Shu, Jingtian Zhou, Qiuyu Lian, Han Li, Dan Zhao, Jianyang Zeng, and Jianzhu Ma.
\newblock Modeling gene regulatory networks using neural network architectures.
\newblock \emph{Nature Computational Science}, 1\penalty0 (7):\penalty0 491--501, 2021.
\newblock ISSN 2662-8457.
\newblock \doi{10.1038/s43588-021-00099-8}.
\newblock URL \url{https://doi.org/10.1038/s43588-021-00099-8}.

\bibitem[Siminelakis et~al.(2019)Siminelakis, Rong, Bailis, Charikar, and Levis]{srb+19}
Paris Siminelakis, Kexin Rong, Peter Bailis, Moses Charikar, and Philip Levis.
\newblock Rehashing kernel evaluation in high dimensions.
\newblock In \emph{International Conference on Machine Learning}, pages 5789--5798. PMLR, 2019.

\bibitem[Singh and Lio'(2019)]{singh2019probabilisticgenerativemodelsharnessing}
Vikash Singh and Pietro Lio'.
\newblock Towards probabilistic generative models harnessing graph neural networks for disease-gene prediction, 2019.
\newblock URL \url{https://arxiv.org/abs/1907.05628}.

\bibitem[Slim et~al.(2022)Slim, Chatelain, Foucauld, and Azencott]{slim2022systematic}
Lotfi Slim, Cl{\'e}ment Chatelain, H{\'e}l{\`e}ne~de Foucauld, and Chlo{\'e}-Agathe Azencott.
\newblock A systematic analysis of gene--gene interaction in multiple sclerosis.
\newblock \emph{BMC Medical Genomics}, 15\penalty0 (1):\penalty0 100, 2022.

\bibitem[Spring and Shrivastava(2021)]{ss21}
Ryan Spring and Anshumali Shrivastava.
\newblock Mutual information estimation using lsh sampling.
\newblock In \emph{Proceedings of the Twenty-Ninth International Conference on International Joint Conferences on Artificial Intelligence}, pages 2807--2815, 2021.

\bibitem[Stuart et~al.(2019)Stuart, Butler, Hoffman, Hafemeister, Papalexi, III, Hao, Stoeckius, Smibert, and Satija]{seurat}
Tim Stuart, Andrew Butler, Paul Hoffman, Christoph Hafemeister, Efthymia Papalexi, William M~Mauck III, Yuhan Hao, Marlon Stoeckius, Peter Smibert, and Rahul Satija.
\newblock Comprehensive integration of single-cell data.
\newblock \emph{Cell}, 177:\penalty0 1888--1902, 2019.
\newblock \doi{10.1016/j.cell.2019.05.031}.
\newblock URL \url{https://doi.org/10.1016/j.cell.2019.05.031}.

\bibitem[Su et~al.(2023)Su, Xu, Shan, Cai, Zhao, and Zhang]{Su2023}
Chang Su, Zichun Xu, Xinning Shan, Biao Cai, Hongyu Zhao, and Jingfei Zhang.
\newblock Cell-type-specific co-expression inference from single cell rna-sequencing data.
\newblock \emph{Nature Communications}, 14\penalty0 (1):\penalty0 4846, Aug 2023.
\newblock ISSN 2041-1723.
\newblock \doi{10.1038/s41467-023-40503-7}.
\newblock URL \url{https://doi.org/10.1038/s41467-023-40503-7}.

\bibitem[Subramanian et~al.(2005)Subramanian, Tamayo, Mootha, Mukherjee, Ebert, Gillette, Paulovich, Pomeroy, Golub, Lander, and Mesirov]{Subramanian2005}
Aravind Subramanian, Pablo Tamayo, Vamsi~K. Mootha, Sayan Mukherjee, Benjamin~L. Ebert, Michael~A. Gillette, Amanda Paulovich, Scott~L. Pomeroy, Todd~R. Golub, Eric~S. Lander, and Jill~P. Mesirov.
\newblock Gene set enrichment analysis: a knowledge-based approach for interpreting genome-wide expression profiles.
\newblock \emph{Proceedings of the National Academy of Sciences of the United States of America}, 102\penalty0 (43):\penalty0 15545--15550, Oct 2005.
\newblock ISSN 0027-8424.
\newblock \doi{10.1073/pnas.0506580102}.
\newblock URL \url{https://www.pnas.org/doi/10.1073/pnas.0506580102}.
\newblock Epub 2005 Sep 30.

\bibitem[Theodoris et~al.(2023)Theodoris, Xiao, Chopra, Chaffin, Al~Sayed, Hill, Mantineo, Brydon, Zeng, Liu, and Ellinor]{Theodoris2023}
Christina~V. Theodoris, Ling Xiao, Anant Chopra, Mark~D. Chaffin, Zeina~R. Al~Sayed, Matthew~C. Hill, Helene Mantineo, Elizabeth~M. Brydon, Zexian Zeng, X.~Shirley Liu, and Patrick~T. Ellinor.
\newblock Transfer learning enables predictions in network biology.
\newblock \emph{Nature}, 618\penalty0 (7965):\penalty0 616--624, 06 2023.
\newblock \doi{10.1038/s41586-023-06139-9}.
\newblock URL \url{https://doi.org/10.1038/s41586-023-06139-9}.

\bibitem[Tsang et~al.(2017)Tsang, Cheng, and Liu]{tsang2017detecting}
Michael Tsang, Dehua Cheng, and Yan Liu.
\newblock Detecting statistical interactions from neural network weights.
\newblock \emph{arXiv preprint arXiv:1705.04977}, 2017.

\bibitem[Uffelmann et~al.(2021)Uffelmann, Huang, Munung, de~Vries, Okada, Martin, Martin, Lappalainen, and Posthuma]{Uffelmann2021}
Emil Uffelmann, Qin~Qin Huang, Nchangwi~Syntia Munung, Jantina de~Vries, Yukinori Okada, Alicia~R. Martin, Hilary~C. Martin, Tuuli Lappalainen, and Danielle Posthuma.
\newblock Genome-wide association studies.
\newblock \emph{Nature Reviews Methods Primers}, 1\penalty0 (1):\penalty0 59, 2021.
\newblock \doi{10.1038/s43586-021-00056-9}.
\newblock URL \url{https://doi.org/10.1038/s43586-021-00056-9}.

\bibitem[Upstill-Goddard et~al.(2013)Upstill-Goddard, Eccles, Fliege, and Collins]{upstill2013machine}
Rosanna Upstill-Goddard, Diana Eccles, Joerg Fliege, and Andrew Collins.
\newblock Machine learning approaches for the discovery of gene--gene interactions in disease data.
\newblock \emph{Briefings in bioinformatics}, 14\penalty0 (2):\penalty0 251--260, 2013.

\bibitem[Vanunu et~al.(2010)Vanunu, Magger, Ruppin, Shlomi, and Sharan]{Vanunu2010}
Orit Vanunu, Oranit Magger, Eytan Ruppin, Tomer Shlomi, and Roded Sharan.
\newblock Associating genes and protein complexes with disease via network propagation.
\newblock \emph{PLoS Computational Biology}, 6:\penalty0 1--9, 2010.
\newblock \doi{10.1371/journal.pcbi.1000641}.
\newblock URL \url{https://doi.org/10.1371/journal.pcbi.1000641}.

\bibitem[Vaswani et~al.(2017)Vaswani, Shazeer, Parmar, Uszkoreit, Jones, Gomez, Kaiser, and Polosukhin]{vaswani2017attention}
Ashish Vaswani, Noam Shazeer, Niki Parmar, Jakob Uszkoreit, Llion Jones, Aidan~N Gomez, {\L}ukasz Kaiser, and Illia Polosukhin.
\newblock Attention is all you need.
\newblock \emph{Advances in neural information processing systems}, 30, 2017.

\bibitem[Wang et~al.(2019)Wang, Tan, Tan, and Yu]{wang2019systems}
Zuo-Teng Wang, Chen-Chen Tan, Lan Tan, and Jin-Tai Yu.
\newblock Systems biology and gene networks in alzheimer’s disease.
\newblock \emph{Neuroscience \& Biobehavioral Reviews}, 96:\penalty0 31--44, 2019.

\bibitem[Wei et~al.(2024)Wei, Islam, Zhou, and Xing]{wei2024self}
Qingyue Wei, Md~Tauhidul Islam, Yuyin Zhou, and Lei Xing.
\newblock Self-supervised deep learning of gene--gene interactions for improved gene expression recovery.
\newblock \emph{Briefings in Bioinformatics}, 25\penalty0 (2):\penalty0 bbae031, 2024.

\bibitem[Williams and Pipkin(2011)]{williams2011genetics}
Paula~J Williams and Fiona~Broughton Pipkin.
\newblock The genetics of pre-eclampsia and other hypertensive disorders of pregnancy.
\newblock \emph{Best practice \& research Clinical obstetrics \& gynaecology}, 25\penalty0 (4):\penalty0 405--417, 2011.

\bibitem[Yuan and Bar-Joseph(2019{\natexlab{a}})]{doi:10.1073/pnas.1911536116}
Ye~Yuan and Ziv Bar-Joseph.
\newblock Deep learning for inferring gene relationships from single-cell expression data.
\newblock \emph{Proceedings of the National Academy of Sciences}, 116\penalty0 (52):\penalty0 27151--27158, 2019{\natexlab{a}}.
\newblock \doi{10.1073/pnas.1911536116}.
\newblock URL \url{https://www.pnas.org/doi/abs/10.1073/pnas.1911536116}.

\bibitem[Yuan and Bar-Joseph(2019{\natexlab{b}})]{pnas.1911536116}
Ye~Yuan and Ziv Bar-Joseph.
\newblock Deep learning for inferring gene relationships from single-cell expression data.
\newblock \emph{Proceedings of the National Academy of Sciences}, 116\penalty0 (52):\penalty0 27151--27158, 2019{\natexlab{b}}.
\newblock \doi{10.1073/pnas.1911536116}.
\newblock URL \url{https://www.pnas.org/doi/abs/10.1073/pnas.1911536116}.

\bibitem[Yuan and Bar-Joseph(2021{\natexlab{a}})]{bbab142}
Ye~Yuan and Ziv Bar-Joseph.
\newblock {Deep learning of gene relationships from single cell time-course expression data}.
\newblock \emph{Briefings in Bioinformatics}, 22\penalty0 (5):\penalty0 bbab142, 04 2021{\natexlab{a}}.
\newblock ISSN 1477-4054.
\newblock \doi{10.1093/bib/bbab142}.
\newblock URL \url{https://doi.org/10.1093/bib/bbab142}.

\bibitem[Yuan and Bar-Joseph(2021{\natexlab{b}})]{yuan2021deep}
Ye~Yuan and Ziv Bar-Joseph.
\newblock Deep learning of gene relationships from single cell time-course expression data.
\newblock \emph{Briefings in bioinformatics}, 22\penalty0 (5):\penalty0 bbab142, 2021{\natexlab{b}}.

\end{thebibliography}

\newpage
\appendix
\section*{Appendix}
\section{More Related Work on Gene-Gene Interaction Discovery}

\label{app:lit_review}

In this section, we provide a more detailed review of the existing work on gene-gene interaction discovery. For gene-gene interaction network construction, genome-wide association studies (GWAS) are widely adopted by biologists to study gene associations using single-nucleotide polymorphisms (SNPs) ~\citep{Uffelmann2021}. However, GWAS have high computational costs and are simply based on direct genotype-phenotype associations instead of wired graph structure. To address this problem, many graphic models have emerged in recent years \citep{bbaa303}. Network-based methods regard gene-gene interaction discovery as a task to construct a homogeneous graph among genes. For example, PRINCE \citep{Vanunu2010} and VAVIEN \citep{Erten2011} apply random work to predict new edges on existing protein-protein interaction (PPI) or gene-gene interaction (GGI) knowledge graphs. VGAE \citep{singh2019probabilisticgenerativemodelsharnessing} and GCAS \citep{Rao2018} explore the potential to incorporate GNN and auto-encoder structure in the GGI network. In addition, existing work like DeepDRIM \citep{Chen2021} and CNNC \citep{pnas.1911536116} successfully improve the construction of GGI through inferencing gene associations on scRNA sequencing data and known transcription factors (TF). While GGI networks and TFs are instrumental for mapping biological processes, they are often plagued by high false-positive rates and context-dependent inaccuracies, especially when derived from large-scale in vitro experiments \citep{bmcbioinformatics2020, jxb2021}. Existing methods that rely on pre-established protein-protein interaction (PPI) or transcription factor (TF) networks are prone to bias because they tend to reinforce known interactions, making it difficult to objectively uncover novel gene-gene interactions. In contrast, our method circumvents this issue by directly discovering GGIs from scRNA-seq data without dependence on prior network knowledge.


\section{Proofs of Theorem~\ref{thm:min-max_density:informal}}

\begin{theorem}[$\minmax$ Density Estimator, formal version of Theorem~\ref{thm:min-max_density:informal}]\label{thm:min-max_density:formal}
    Given a cell dataset $X$, for every $q\in X$, we compute $w_q$ following Algorithm~\ref{alg:two_pass_diverse}. Next, we have
    \begin{align*}
        \E[w_q] = \sum_{x\in X} (\minmax(x,q)+o(1)),
    \end{align*}
    where $\minmax$ is the $\minmax$ similarity defined in Definition~\ref{def:min-max}.
    As a result, $w_q$ is an estimator for $\minmax$ density $\mathcal{K}(q)$ defined in Definition~\ref{def:min-max_density} with $\phi(q,x)=\minmax(x,q)+o(1)$.
\end{theorem}

\begin{proof}
    According to Theorem 2 in \citep{coleman2019race}, the expectation of $w_q$ should be:
    \begin{align*}
        \E[w_q]=\sum_{x\in X} \Pr_{h\sim \mathcal{H}}[h(q)=h(x)]
    \end{align*}

    According to Definition~\ref{def:cws}, we have
    \begin{align*}
        \Pr_{h\sim \mathcal{H}}[h(q)=h(x)] = \minmax(x,q)+o(1)
    \end{align*}.

    As a result, 
    \begin{align*}
        \E[w_q]=\sum_{x\in X} (\minmax(x,q)+o(1))
    \end{align*}

    Moreover, since $\minmax(x,q)+o(1)$ is a monotonic increasing function of $\minmax(x,q)$. We say that $w_q$ is an estimator for $\minmax$ density $\mathcal{K}(q)$ defined in Definition~\ref{def:min-max_density} with $\phi(q,x)=\minmax(x,q)+o(1)$.
\end{proof}


\section{Experiment Details}
\subsection{Model Implementations}\label{app:exp_details}


\textbf{Transformer Configurations:} In this work, we used the standard multi-head self-attention introduced in \citep{vaswani2017attention}. We do not see the potential of the proposed blocks in ~\citep{lee2019set} in our setting. Moreover, we perform padding on each batch of training and inference of single-cell data. Accordingly, we introduce a padding mask in the attention mechanism to avoid computation on the padded position. For each input sequence, we represent them as embedding by a lookup table that maps a vocabulary of 36,601 genes to 128-dimensional vectors. Subsequently, the embedded data passes through 4 transformer encoder blocks. Each encoder block features 8 attention heads, to capture complex, non-linear relationships within the data. Finally, the output is fed into a linear layer that classifies the data labels. Here the label for the cell can be disease-oriented, such as whether this cell is from an Alzheimer's disease patient. We represent each input sequence by employing a lookup table that transforms a comprehensive vocabulary of 36,601 genes into 128-dimensional embedding vectors. These vectors are subsequently processed through a series of 4 Transformer encoder blocks. Each encoder block is equipped with 8 attention heads, a 512-dimensional feedforward layer, and a dropout layer in a ratio of 0.1. The processed outputs are then directed to a linear classification layer, which is tasked with predicting labels indicative of Alzheimer's disease conditions. We adopted the Adam Optimization Algorithm to minimize the loss function ~\cite{kingma2017adam}. The model is trained under a learning rate of 1e-5 and the batch size of our data-loader is set as 128. The testing results for the transformer after 3 epochs of training are given in Table \ref{tab:model_performance}.


\textbf{MLP Configurations:} The MLP consists of 2 hidden layers, with 128 and 256 hidden units respectively. Each hidden layer is followed by a dropout and a Softplus module. The MLP is trained under a learning rate of 1e-4 and the batch size of our data-loader is set as 128. We adopted the Adam Optimization Algorithm to minimize the loss function ~\cite{kingma2017adam}. The testing results for the MLP after 80 epochs of training are given in Table \ref{tab:model_performance}.

\begin{table}[h]
\caption{Complete Performance comparison of models on neuronal cell data.}
\label{tab:model_performance_complete}
\centering
\begin{tabular}{cccc}
\toprule
\textbf{Model} & \textbf{Training Dataset} & \textbf{F1 Score} & \textbf{Accuracy} \\
\midrule
\multirow{7}{*}{MLP} & Pax6 & 78.91 & 82.71 \\
 & L5\_ET & 62.02 & 73.31 \\
 & L6\_CT & 91.14 & 92.01 \\
 & L6\_IT\_Car3 & 95.34 & 95.51 \\
 & L6b & 86.01 & 88.76 \\
 & Chandelier & 81.66 & 84.56 \\
 & L5\_6\_NP & 89.33 & 90.42 \\
 & All Neuronal Cell Types & 97.23 & 97.25 \\
\midrule
CelluFormer & All Neuronal Cell Types & \textbf{98.12} & \textbf{98.12} \\
scGPT & All Neuronal Cell Types & {93.85} & {94.32} \\
scFoundation & All Neuronal Cell Types & {97.38} & {97.39} \\
\bottomrule
\end{tabular}%
\end{table}

\textbf{Fine-tuning configurations for scFoundatoin and scGPT:}
For fine-tuning scGPT, we use an LR of 1e-4 and a batch size of 64. We utilize a step scheduler down to 90\% of the original learning rate every 10 steps. The training process converges after 6 epochs. For scFoundation, we use an LR of 1e-4 and a batch size of 32. We fine-tune scFoundation for 10 epochs. The performances of scFoundation and scGPT on classifying disease cells are shown in Table \ref{tab:model_performance_complete}.

\textbf{Implementation and Computation Resources:} Our codebase and workflow are implemented in PyTorch \cite{paszke2017automatic}. We trained and tested our workflow on a server with 8 Nvidia Tesla V100 GPU and a 44-core/88-thread processor (Intel(R) Xeon(R) CPU E5-2699A v4 @ 2.40GHz).

\subsection{Evaluation Metrics}

\label{app:eval} The normalized enrichment score (NES) is the main metric used to analyze gene set enrichment outcomes \cite{Subramanian2005}. This score quantifies the extent of over-representation of a ground truth dataset at the top of the ranked list of gene-gene interactions. That is, the higher the better. We can calculate NES by starting at the top of the ranked list and moving through it, adjusting a running tally by increasing the score for each gene-gene interaction in the ground truth dataset and decreasing it for others based on each gene-gene interaction's rank. This process continues until we evaluate the entire ranked list to identify the peak score, which is the enrichment score. The BioGRID Dataset provides human protein/genetic interactions. Specifically, \textit{BioGRID} contributes $204,831$ protein/genetic interactions that help verify the enrichment of genuine biological interactions in a ranked list of gene-gene interactions. DisGenet contains 429,036 gene-disease associations (GDAs), connecting 17,381 genes to 15,093 diseases, disorders, and abnormal human phenotypes \cite{Oughtred2019, gkw943}.
\section{More Experments}

\subsection{Contrastive Ranking} \label{app:more_exp}

Here, we also explore alternative strategies for aggregating attention maps. While Pearson Correlation, Spearman's Correlation, and CS-CORE themselves cannot capture the information between gene pairs the the target disease, we believe Transformers learn the difference among data with varying labels. Hence, we do not need to calculate the difference between attention maps aggregated on data with varying labels. However, given that the Transformer is trained to classify disease cells, we hypothesize that it likely assigns significant attention to specific gene pairs within disease cells. To evaluate this, we applied our pipeline to three distinct datasets. The experimental results summarized in Table \ref{tab:aggregate_exp} show that our pipeline achieves improved NES when both disease and non-disease cells are used as inputs. These findings suggest that the Transformer benefits from data both positive and negative labels to provide a more comprehensive understanding of features. 

\begin{table}[ht]
\caption{This experiment involves three groups. In the first group, the Transformer only takes the disease cells for inference. We directly evaluate the ranked list given by aggregated attention map across disease cells. In the second group, we calculate the aggregated attention maps on the disease cells and the non-disease cells respectively. The final attention map is obtained by subtracting these two attention maps. The third group is to aggregate attention maps across the whole dataset.}
\label{tab:aggregate_exp}
\centering
\resizebox{0.9\columnwidth}{!}{%
\begin{tabular}{cccccccc}
\toprule
Strategy & L5\_ET & L6\_CT & Pax6 & L5\_6\_NP & L6b & Chandelier & L6\_IT\_Car3 \\ \midrule
AD cells & 1.09 & 1.09 & 0.98 & 0.78 & 1.13 & 0.90 & 0.89 \\
AD cells - Non-AD cells & 1.08 & 0.89 & 1.05 & 0.76 & 0.82 & 0.65 & \textbf{1.39} \\
All cells & \textbf{1.15} & \textbf{1.18} & \textbf{1.25} & \textbf{1.21} & \textbf{1.13} & \textbf{1.17} & 1.22 \\ \bottomrule
\end{tabular}%
}
\end{table}

\subsection{Empirical Study on Parameter \texorpdfstring{$R$}{~} in Algorithm~\ref{alg:two_pass_diverse}}\label{sec:exp_param_r}

\begin{table}[ht]
\caption{The Mean value of NES results across 5 experiments on L5\_ET, L6\_CT, and Pax6 cell type datasets.}
\label{tab:Mean_hash_row}
\centering
\resizebox{0.9\columnwidth}{!}{%
\begin{tabular}{lcrrrr}
\toprule
\multicolumn{1}{c}{\multirow{2}{*}{Dataset}} & \multirow{2}{*}{Sample Size} & \multicolumn{4}{c}{Mean of NES} \\
\multicolumn{1}{c}{} &  & Uniform & WDS with R=100 & WDS with R=200 & WDS with R=500 \\ \midrule
\multirow{4}{*}{L5\_ET} & 1\% & 0.90 & \textbf{1.02} & 0.95 & 0.93 \\
 & 2\% & 0.89 & \textbf{1.17} & \textbf{1.17} & 0.97 \\
 & 5\% & 1.02 & 0.97 & \textbf{1.19} & 1.11 \\
 & 10\% & 0.87 & 1.01 & \textbf{1.07} & \textbf{1.07} \\ \midrule
\multirow{4}{*}{L6\_CT} & 1\% & 0.85 & \textbf{1.19} & \textbf{1.19} & 1.11 \\
 & 2\% & 1.05 & \textbf{1.21} & 1.18 & 1.09 \\
 & 5\% & 0.93 & 1.13 & \textbf{1.23} & 1.21 \\
 & 10\% & 0.91 & \textbf{1.23} & 1.21 & 1.20 \\ \midrule
\multirow{4}{*}{Pax6} & 1\% & 0.94 & \textbf{1.13} & 1.08 & 1.17 \\
 & 2\% & 1.03 & \textbf{1.22} & 1.18 & 1.19 \\
 & 5\% & 0.98 & \textbf{1.21} & 1.20 & 1.19 \\
 & 10\% & 1.06 & 1.19 & 1.17 & \textbf{1.22} \\
\bottomrule
\end{tabular}%
}
\end{table}

During our experiments on WDS, we observed that the value of $R$ (see Algorithm~\ref{alg:two_pass_diverse}) has a noticeable impact on NES performance. In Table \ref{tab:Mean_hash_row} and Table \ref{tab:MSE_hash_row}, we evaluate three different $R$ values ranging from 100 to 500. The results demonstrate that increasing $R$ leads to a significant decline in NES. Although WDS with smaller $R$ values yields relatively higher NES, it tends to diverge from the NES calculated on the entire dataset.

\begin{table}[ht]
\caption{The MSE of NES results across 5 experiments on L5\_ET, L6\_CT, and Pax6 cell type datasets. The MSE values are calculated according to the results in Table \ref{tab:RQ1_res}.}
\label{tab:MSE_hash_row}
\centering
\resizebox{0.9\columnwidth}{!}{%
\begin{tabular}{lcrrrr}
\toprule
\multicolumn{1}{c}{\multirow{2}{*}{Dataset}} & \multirow{2}{*}{Sample Size} & \multicolumn{4}{c}{MSE of NES} \\
\multicolumn{1}{c}{} &  & Uniform & WDS with R=100 & WDS with R=200 & WDS with R=500 \\ \midrule
\multirow{4}{*}{L5\_ET} & 1\% & 0.0636 & 0.0408 & \textbf{0.0178} & 0.0477 \\
 & 2\% & 0.0653 & 0.0005 & \textbf{0.0004} & 0.0339 \\
 & 5\% & 0.0181 & \textbf{0.0014} & 0.0310 & 0.0018 \\
 & 10\% & 0.0790 & \textbf{0.0062} & 0.0192 & 0.0064 \\ \midrule
\multirow{4}{*}{L6\_CT} & 1\% & 0.1033 & 0.0002 & 0.0002 & 0.0046 \\
 & 2\% & 0.0151 & \textbf{0.0001} & 0.0014 & 0.0070 \\
 & 5\% & 0.0610 & 0.0028 & \textbf{0.0025} & 0.0013 \\
 & 10\% & 0.0681 & 0.0011 & 0.0031 & \textbf{0.0007} \\ \midrule
\multirow{4}{*}{Pax6} & 1\% & 0.0920 & 0.0264 & 0.0135 & \textbf{0.0057} \\
 & 2\% & 0.0488 & 0.0047 & \textbf{0.0006} & 0.0027 \\
 & 5\% & 0.0695 & 0.0022 & \textbf{0.0015} & 0.0027 \\
 & 10\% & 0.0362 & 0.0058 & 0.0028 & \textbf{0.0008} \\ \bottomrule
\end{tabular}%
}
\end{table}

\end{document}